\documentclass[]{article}
\usepackage{amssymb,amsmath,multirow,graphicx}
\usepackage[paper=a4paper,textwidth=160mm]{geometry}
\usepackage[ruled,vlined,linesnumbered]{algorithm2e}
\usepackage[noend]{algpseudocode}
\makeatletter
\def\BState{\State\hskip-\ALG@thistlm}
\makeatother

\usepackage{dsfont}

\def\Z{\ensuremath{\mathds{Z}}}
\def\R{\ensuremath{\mathds{R}}}
\def\E{\ensuremath{\mathds{E}}}

\newcommand{\ud}{\mathrm{d}}

\newcommand{\ie}{\textit{i.e.},}

\newtheorem{theorem}{Theorem}[section]
\newtheorem{remark}{Remark}[section]
\newtheorem{definition}{Definition}[section]

\newtheorem{proposition}[theorem]{Proposition}

\newenvironment{proof}[1][Proof]{\begin{trivlist}
\item[\hskip \labelsep {\bfseries #1}]}{\end{trivlist}}

\newcommand{\qed}{\nobreak \ifvmode \relax \else
      \ifdim\lastskip<1.5em \hskip-\lastskip
      \hskip1.5em plus0em minus0.5em \fi \nobreak
      \vrule height0.75em width0.5em depth0.25em\fi}

\usepackage[]{hyperref}
\hypersetup{
pdftitle={Riemannian Manifolds Extension of Mathematical Morphology},
colorlinks=true,citecolor=blue,linkcolor=red
}
\ifpdf
\hypersetup{
  pdfauthor={E. H. S. Diop, T. Fall, A. Mbengue and M. Daoudi}
}
\fi

\usepackage{dsfont}
\usepackage{textcomp}

\def\Z{\ensuremath{\mathds{Z}}}
\def\R{\ensuremath{\mathds{R}}}
\def\E{\ensuremath{\mathds{E}}}

\usepackage{lipsum}
\usepackage{amsfonts,amssymb}
\usepackage{graphicx}
\usepackage{epstopdf}
\usepackage{subcaption}

\graphicspath{{results/}} 

\usepackage{caption}
\usepackage{subcaption}
\usepackage{tikz}
\usepackage{pgfplots}
\usetikzlibrary{spy,calc}

\usepackage{multirow}
\usepackage{ctable}

\newif\ifblackandwhitecycle
\gdef\patternnumber{0}

\pgfkeys{/tikz/.cd,
    zoombox paths/.style={
        draw=orange,
        very thick
    },
    black and white/.is choice,
    black and white/.default=static,
    black and white/static/.style={ 
        draw=white,   
        zoombox paths/.append style={
            draw=white,
            postaction={
                draw=black,
                loosely dashed
            }
        }
    },
    black and white/static/.code={
        \gdef\patternnumber{1}
    },
    black and white/cycle/.code={
        \blackandwhitecycletrue
        \gdef\patternnumber{1}
    },
    black and white pattern/.is choice,
    black and white pattern/0/.style={},
    black and white pattern/1/.style={    
            draw=white,
            postaction={
                draw=black,
                dash pattern=on 2pt off 2pt
            }
    },
    black and white pattern/2/.style={    
            draw=white,
            postaction={
                draw=black,
                dash pattern=on 4pt off 4pt
            }
    },
    black and white pattern/3/.style={    
            draw=white,
            postaction={
                draw=black,
                dash pattern=on 4pt off 4pt on 1pt off 4pt
            }
    },
    black and white pattern/4/.style={    
            draw=white,
            postaction={
                draw=black,
                dash pattern=on 4pt off 2pt on 2 pt off 2pt on 2 pt off 2pt
            }
    },
    zoomboxarray inner gap/.initial=5pt,
    zoomboxarray columns/.initial=2,
    zoomboxarray rows/.initial=2,
    subfigurename/.initial={},
    figurename/.initial={zoombox},
    zoomboxarray/.style={
        execute at begin picture={
            \begin{scope}[
                spy using outlines={%
                    zoombox paths,
                    width=\imagewidth / \pgfkeysvalueof{/tikz/zoomboxarray columns} - (\pgfkeysvalueof{/tikz/zoomboxarray columns} - 1) / \pgfkeysvalueof{/tikz/zoomboxarray columns} * \pgfkeysvalueof{/tikz/zoomboxarray inner gap} -\pgflinewidth,
                    height=\imageheight / \pgfkeysvalueof{/tikz/zoomboxarray rows} - (\pgfkeysvalueof{/tikz/zoomboxarray rows} - 1) / \pgfkeysvalueof{/tikz/zoomboxarray rows} * \pgfkeysvalueof{/tikz/zoomboxarray inner gap}-\pgflinewidth,
                    magnification=3,
                    every spy on node/.style={
                        zoombox paths
                    },
                    every spy in node/.style={
                        zoombox paths
                    }
                }
            ]
        },
        execute at end picture={
            \end{scope}
            \node at (image.north) [anchor=north,inner sep=0pt] {\subcaptionbox{\label{\pgfkeysvalueof{/tikz/figurename}-image}}{\phantomimage}};
            \node at (zoomboxes container.north) [anchor=north,inner sep=0pt] {\subcaptionbox{\label{\pgfkeysvalueof{/tikz/figurename}-zoom}}{\phantomimage}};
     \gdef\patternnumber{0}
        },
        spymargin/.initial=0.5em,
        zoomboxes xshift/.initial=1,
        zoomboxes right/.code=\pgfkeys{/tikz/zoomboxes xshift=1},
        zoomboxes left/.code=\pgfkeys{/tikz/zoomboxes xshift=-1},
        zoomboxes yshift/.initial=0,
        zoomboxes above/.code={
            \pgfkeys{/tikz/zoomboxes yshift=1},
            \pgfkeys{/tikz/zoomboxes xshift=0}
        },
        zoomboxes below/.code={
            \pgfkeys{/tikz/zoomboxes yshift=-1},
            \pgfkeys{/tikz/zoomboxes xshift=0}
        },
        caption margin/.initial=4ex,
    },
    adjust caption spacing/.code={},
    image container/.style={
        inner sep=0pt,
        at=(image.north),
        anchor=north,
        adjust caption spacing
    },
    zoomboxes container/.style={
        inner sep=0pt,
        at=(image.north),
        anchor=north,
        name=zoomboxes container,
        xshift=\pgfkeysvalueof{/tikz/zoomboxes xshift}*(\imagewidth+\pgfkeysvalueof{/tikz/spymargin}),
        yshift=\pgfkeysvalueof{/tikz/zoomboxes yshift}*(\imageheight+\pgfkeysvalueof{/tikz/spymargin}+\pgfkeysvalueof{/tikz/caption margin}),
        adjust caption spacing
    },
    calculate dimensions/.code={
        \pgfpointdiff{\pgfpointanchor{image}{south west} }{\pgfpointanchor{image}{north east} }
        \pgfgetlastxy{\imagewidth}{\imageheight}
        \global\let\imagewidth=\imagewidth
        \global\let\imageheight=\imageheight
        \gdef\columncount{1}
        \gdef\rowcount{1}
        
    },
    image node/.style={
        inner sep=0pt,
        name=image,
        anchor=south west,
        append after command={
            [calculate dimensions]
            node [image container,subfigurename=\pgfkeysvalueof{/tikz/figurename}-image] {\phantomimage}
            node [zoomboxes container,subfigurename=\pgfkeysvalueof{/tikz/figurename}-zoom] {\phantomimage}
        }
    },
    color code/.style={
        zoombox paths/.append style={draw=#1}
    },
    connect zoomboxes/.style={
    spy connection path={\draw[draw=none,zoombox paths] (tikzspyonnode) -- (tikzspyinnode);}
    },
    help grid code/.code={
        \begin{scope}[
                x={(image.south east)},
                y={(image.north west)},
                font=\footnotesize,
                help lines,
                overlay
            ]
            \foreach \x in {0,1,...,9} { 
                \draw(\x/10,0) -- (\x/10,1);
                \node [anchor=north] at (\x/10,0) {0.\x};
            }
            \foreach \y in {0,1,...,9} {
                \draw(0,\y/10) -- (1,\y/10);                        \node [anchor=east] at (0,\y/10) {0.\y};
            }
        \end{scope}    
    },
    help grid/.style={
        append after command={
            [help grid code]
        }
    },
}

\newcommand\phantomimage{%
    \phantom{%
        \rule{\imagewidth}{\imageheight}%
    }%
}
\newcommand\zoombox[2][]{
    \begin{scope}[zoombox paths]
        \pgfmathsetmacro\xpos{
            (\columncount-1)*(\imagewidth / \pgfkeysvalueof{/tikz/zoomboxarray columns} + \pgfkeysvalueof{/tikz/zoomboxarray inner gap} / \pgfkeysvalueof{/tikz/zoomboxarray columns} ) + \pgflinewidth
        }
        \pgfmathsetmacro\ypos{
            (\rowcount-1)*( \imageheight / \pgfkeysvalueof{/tikz/zoomboxarray rows} + \pgfkeysvalueof{/tikz/zoomboxarray inner gap} / \pgfkeysvalueof{/tikz/zoomboxarray rows} ) + 0.5*\pgflinewidth
        }
        \edef\dospy{\noexpand\spy [
            #1,
            zoombox paths/.append style={
                black and white pattern=\patternnumber
            },
            every spy on node/.append style={#1},
            x=\imagewidth,
            y=\imageheight
        ] on (#2) in node [anchor=north west] at ($(zoomboxes container.north west)+(\xpos pt,-\ypos pt)$);}
        \dospy
        \pgfmathtruncatemacro\pgfmathresult{ifthenelse(\columncount==\pgfkeysvalueof{/tikz/zoomboxarray columns},\rowcount+1,\rowcount)}
        \global\let\rowcount=\pgfmathresult
        \pgfmathtruncatemacro\pgfmathresult{ifthenelse(\columncount==\pgfkeysvalueof{/tikz/zoomboxarray columns},1,\columncount+1)}
        \global\let\columncount=\pgfmathresult
        \ifblackandwhitecycle
            \pgfmathtruncatemacro{\newpatternnumber}{\patternnumber+1}
            \global\edef\patternnumber{\newpatternnumber}
        \fi
    \end{scope}
}

\begin{document}

\title{Geometric Generative Models based on Morphological Equivariant PDEs and GANs}

\author{
El Hadji S. Diop\footnote{NAGIP-Nonlinear Analysis and Geometric Information Processing Group, Department of Mathematics, University Iba Der Thiam of Thies, Thies BP 967, Senegal. E-mails: ehsdiop@hotmail.com, thiernofall571@gmail.com}
\and 	%
Thierno Fall$^{\star}$
\and 	%
Alioune Mbengue$^{\star}$\footnote{Department of Mathematics and Computer Science, University Cheikh Anta Diop, Senegal. E-mail: 99aliou@gmail.com}
\and 	Mohamed Daoudi\thanks{IMT Nord Europe, Institut Mines-Telecom, Univ. Lille, Centre for Digital Systems, F-59000 Lille, France
Univ. Lille, CNRS, Centrale Lille, Institut Mines-Telecom, UMR 9189 CRIStAL, F-59000 Lille, France. E-mail: mohamed.daoudi@imt-nord-europe.fr}}

\date{}

\maketitle

\begin{abstract}
Content and image generation consist in creating or generating data from noisy information by extracting specific features such as texture, edges, and other thin image structures. This work deals with image generation, two main problems are addressed: ({\it i} ) improvements of specific feature extraction while accounting at multiscale levels intrinsic geometric features, and ({\it ii} ) equivariance of the network for reducing the complexity and providing a geometric interpretability. We propose a geometric generative model based on an equivariant partial differential equation (PDE) for group convolution neural networks (G-CNNs), so called PDE-G-CNNs, built on morphology operators and generative adversarial networks (GANs). Equivariant morphological PDE layers are composed of multiscale dilations and erosions formulated in Riemannian manifolds, while group symmetries are defined on a Lie group. We take advantage of the Lie group structure to properly integrate the equivariance in layers, and use the Riemannian metric to solve the multiscale morphological operations. Each element of the Lie group is associated with a unique point in the manifold, which helps us derive a Riemannian metric from a tensor field invariant under the Lie group so that the induced metric shares the same symmetries. The proposed geometric morphological GAN model, termed as GM-GAN, is obtained thanks to morphological equivariant convolutions in PDE-G-CNNs. GM-GAN is evaluated qualitatively and quantitatively using FID on MNIST and RotoMNIST, preliminary results show noticeable improvements compared classical GAN.

\end{abstract}

\section{Introduction}

Content generation is one of the most quickly developing domain, mainly because of its potential real life applications. Encouraging  results of generative models are due to prominent advances in learning methods based on adversarial neural network. Generative models are particularly interesting because of their ability to create or reject samples outside the training set. This capability to generate data beyond mere density estimation makes generative models become very important for the prediction of samples outside the training set, and may be a reason of their high interests in recent years. Generative models also have found many interesting real life applications in various domains; for instance, in realistic synthetic images generation, content generation from words and phrases \cite{Reed2016,Zhang2017}, adversarial training \cite{Ren2018}, missing data completion \cite{Yeh2017,Iizuka2017,Yu2018}, image manipulation based on predefined features \cite{Gauthier2014,Radford2015,Chen2016,Liu2017,Isola2017,Zhu2017}, multimodal tasks with a single input \cite{Hausman2017,Vukotic2017}, samples generation from the same distribution \cite{Antoniou2017,Xu2019}, data quality enhancement \cite{Ledig2017,Sajjadi2017,Xu2017}. GANs \cite{Goodfellow2014,Goodfellow2017} brought a new perspective to the deep learning community, deep learning with adversarial training is considered today as one of the most robust technique. With adversarial generative networks, there exists not only a good neural network-based classifier, referred to as the discriminator network, but also a generative network capable of producing realistic adversarial samples, all within a single architecture. This means that we now have a network that is aware of internal representations through its training to distinguish real inputs from artificial ones. Many extensions have been built for increasing its performances. Conditional GAN (CGAN) \cite{Gauthier2014} was proposed as an extension of original GAN for generating facial images on the basis of facial attributes. Deep Convolutional GAN (DCGAN) \cite{Radford2015} was proposed for image generation where both the generator and discriminator networks are convolutional. GRAN \cite{Im2016} is a GAN model based on a sequential process. Bidirectional GAN (BiGAN) and extensions \cite{Donahue2016,Chen2017} were proposed to map data into a latent code similar to an autoencoder. Generative Multi-Adversarial Network (GMAN) \cite{Durugkar2016} was proposed for extending the minimax game to multiple players in GANs. In a different perspective, Wasserstein Generative Adversarial Network (WGAN) \cite{Arjovsky2017} was introduced to reduce the instability problems that occur during the training step, and also to eliminate the mode collapse effect. GANs and variants lack an inference mechanism. 

\paragraph{Related works} 
Significant advances in deep learning progress are attributed to CNNs \cite{Gu2018}. Despite its successful applications in many real life problems, it is still not very clear why deep learning techniques work. Pursuing this goal, many works attempt to give an answer to this so challenging question by setting mathematical frameworks that underlie the process. A promising direction is to consider symmetries as a fundamental design principle for network architectures. This can be implemented by constructing deep neural networks that are compatible with a symmetry group that acts transitively on the input data. Among noticeable properties in CNNs, the equivariance concerning translations played an important role. Equivariance means that the operation of performing a transformation of the input data then passing them through the network is the same as passing the input data through the network and then performing a transformation of the output. CNNs are inherently translationally invariant; however, invariance does not extend straightforward to other types of transformations. G-CNNs \cite{Cohen2016,Bekkers2018,Cohen2019} were introduced to tackle this issue by generalizing CNNs in a way such that symmetries are incorporated and fully exploited in the learning process. In addition to reducing a lot sample complexity by exploiting symmetries since there is no more need to learn them, G-CNNs show great improvements compared to former CNNs \cite{Winkels2018,Cohen2018,Bekkers2019}. Very recently, PDE-G-CNNs \cite{Smets13July2022,Bellaard2023} were proposed as PDEs-based framework based that generalized G-CNNs. Authors proposed to replace the classical convolution, pooling and ReLUs in traditional CNNs by resolving a PDE composed of four terms where each one behaved separately like a convection, diffusion, dilation and erosion. 
The proposed PDEs were solved by providing analytical kernels approximations \cite{Smets13July2022} and exact kernels sub-Riemannian approximations \cite{Bellaard2023}. PDE-G-CNNs were shown to increase the performance for classification tasks. Intensive research on equivariant operators other than transformations is still conducted \cite{Romero2020,Gerken2023,Tian2024}. 


\paragraph{Paper contributions} We provide here noticeable improvements of former GAN models by using a geometric approach based on equivariant operators defined in a Lie group, and on mathematical morphology formulated in Riemannian manifolds. Indeed, we propose to introduce non-linearities into classical GANs by means of group-equivariant morphological operators. Generative models aim at creating or generating data from noisy information by extracting specific features such as texture, edges, and other thin image structures. In this study, we are interested in two main problems: {\bf 1)} improving of specific feature extraction while accounting at multiscale levels intrinsic geometric features, and {\bf 2)} making the network equivariant for reducing its complexity and providing a geometric interpretability. As for alternatives for these issues, we propose here a new geometric generative model based on a new PDE-G-CNNs built on morphology operators, geometric image processing techniques \cite{DubrovinaKarni2014,Fadili2015a,Burger2016,Welk2019a} and GANs. Mathematical morphology (MM) \cite{Soille1999} has been efficiently applied in multiscale image and data analysis \cite{Soille1999,Najman2010,Shih2009} and in various CNNs architectures \cite{Mellouli2019,Roy2021,Duits2021,Sangalli2022}. The functional analysis formulation \cite{Alvarez1993,Akian1994,Schmidt2016} was an interesting way for linking MM to first order Hamilton-Jacobi PDEs and scale-spaces. The proposed PDE-G-CNNs is designed in a way such that morphological PDE layers are the multiscale dilations and erosions formulated in Riemannian manifolds, and symmetries are defined on a Lie group. Riemannian based techniques are well proven to noticeably improve on Euclidean based ones in image and data representations and analysis; namely, in video surveillance, shape and surface analysis, human body and face analysis, image segmentation \cite{Su2014,Balan2015,Citti2016,Kurtek2016,Younes2019,Pierson2022a}. Working on Lie groups lets us take advantage of the group structure for properly integrating the equivariance property through layers, on one hand, and be able to use the Riemannian metric to solve morphological dilations and erosions obtained as viscosity solutions of first order Hamilton-Jacobi PDEs and given by Hopf-Lax-Oleinik formulas in Riemannian manifolds \cite{Diop2021}, on the other hand. In addition, we associate to each point in the group a point in the manifold, and derive a metric on the Riemannian manifold from a tensor field invariant under the Lie group so that the induced metric shares the same symmetries. Also, there is no more need to approximate the kernels, since we choose to work on the hyperbolic ball yielding an explicit computation of the geodesic distances, and so a compact formulation of the structuring functions or kernels in more general forms that the canonical ones. 

\paragraph{Manuscript organization} In Section~\ref{sect2}, we recall some background on multiscale mathematical operators and their links with PDEs. In Section~\ref{sect3}, we define the notion of equivariance in Lie groups	and present the group invariance property on Riemannian manifolds. In Section~\ref{sect4}, we present the viscosity solutions for morphological dilations and erosions formulated as Lie group morphological convolutions in Riemannian manifolds. The proposed geometric generative (GM-GAN) model in presented in Section~\ref{sect5}. Section~\ref{sect6} is dedicated to numerical experiments and comparisons with classical GAN models. The paper ends in Section~\ref{sect7} where concluding remarks and perspectives are discussed.


\section{Background on PDEs-based mathematical morphology}\label{sect2}

Let $b: \R^2\rightarrow\bar\R$ be a concave function, known also as the structuring function or convolution kernel. Let us consider the subset $\E$ of $\Z^2$ and the function $f:\E \rightarrow \bar\R$. 
\begin{definition}
\label{def:morpho}
Morphological dilation and erosion are respectively defined as:
\begin{align}
& f \oplus b(x) = \sup_{y\in\E}[f(y) + b(x-y)] \label{inf_sup:dil} \\
& f\ominus b(x) = \inf_{y\in\E}[f(y) - b(y-x)]. \label{inf_sup:ero}
\end{align}
\end{definition}
Let $B \subseteq \E$ be a bounded set. A flat structuring function (SF) satisfies $b(x)= 0$ if $x\in B$ and $b(x)=-\infty$ if $x\in B^c$. The flat morphological dilation and erosion respectively write:
\begin{equation}
\label{inf_sup:st}
f\oplus B(x) = \sup_{y\in B}[f(x-y)] \mbox{ and } f\ominus B(x) = \inf_{y\in B}[f(x+y)].
\end{equation}
As for an interpretation, erosion shrinks positive peaks, and peaks thinner than the structuring function disappear. One has the dual effects for morphological flat dilation. Both the morphological dilation and erosion are translation invariant.
\begin{definition}
\label{def:increas}
Let $\mathcal{F}$ be a family of real functions defined on $\Omega\subseteq\R^2$. We say that $T:\mathcal{F}\rightarrow \mathcal{F}$ is said to be increasing (monotone) if and only if it satisfies:\\$\forall~f_1, f_2\in\mathcal{F}$ such that $(f_1 \geq f_2$ on $\Omega)$ implies $(T(f_1) \geq T(f_2)$ on $\Omega)$.
\end{definition}
\begin{proposition}
Morphological dilation and erosion satisfy the following duality and adjunction properties:
\begin{enumerate}
 \item duality: $f\oplus b= - (-f\ominus b)$
 \item adjunction: $(f_1\oplus b \leq f_2$ on $E)$ $\Longleftrightarrow (f_1 \leq f_2\ominus b$ on $E)$.
\end{enumerate}
\end{proposition}

Let $(b_t)_{t\geq 0}$ the family of structuring functions defined by using the SF $b$, as follows:
\begin{equation*}
b_{t}(x) = \left\{
\begin{array}{cl}
  t b(x/t) & \text{for } t>0 \\
  0        & \text{for } t=0,\; x=0 \\
  -\infty  & \text{otherwise}.
\end{array}
\right.
\end{equation*}
The family $(b_t)_{t\geq0}$ satisfies the semi-group property:\\ $\forall~s,t\geq 0$, $(b_{s}\oplus b_{t})(x) =$ $b_{s+t}(x,y)$. 
\begin{definition}
Morphological multiscale dilations and erosions are defined as follows:
\begin{align}
& (f\oplus b_t)(x) = \sup_{y\in\E}[f(y) + b_t(x-y)] \label{scale_inf_sup:dil} \\
& (f\ominus b_t)(x) = \inf_{y\in\E}[f(y) - b_t(y-x)]. \label{scale_inf_sup:ero}
\end{align}
\end{definition}
Considering flat structuring function (SF), morphological multiscale dilations and erosions are obtained equivalently  by considering $B_t = tB$ as multiscale SFs. \\

Linking between morphological scale-spaces and PDEs is established \cite{Alvarez1993,Sapiro1993,Brockett1994} by running the following PDE for performing multiscale flat dilations and erosions on an image $f$: 
\begin{equation}
\label{cont}
 \partial_t w \pm\|\nabla u\| = 0;~w(\cdotp,0) = f.
\end{equation}
Depending on the shape of SF, different PDEs can be obtained. For instance, considering the sets\\ $S_p= \left\{x=(x_1,x_2)\in \R^2: |x|_p \leq 1\right\}$, where $|\cdotp|_p$ is the $L^p$ norm, one gets:
\begin{itemize}
\item for a square $S_1$: $\partial_t w \pm\| \nabla u\|_1 = 0;~u(\cdotp,0) = f$ 
\item for a dis $S_2$: $\partial_t w \pm\| \nabla u\|_2 = 0;~u(\cdotp,0) = f$
\item for a rhombus $S_\infty$: $\partial_t w \pm\| \nabla u\|_{\infty} = 0;~u(\cdotp,0) = f$.
\end{itemize}
Notice that PDE \eqref{cont} is a special case of first order Hamilton-Jacobi equation type, which can be formulated in a more general form as follows:
\begin{equation}
\label{pde_hamil}
 \left\{
\begin{array}{l} 
 \dfrac{\partial w(x,t)}{\partial t} + H
 \left(x,\nabla w(x,t)\right) = 0  \mbox{ on }  \R^n \times (0,+\infty)  \\
w(\cdotp,0) = f \mbox{ on } \R^n.
\end{array}
 \right.
\end{equation}
General Hamilton-Jacobi equation is studied in a viscosity sense \cite{Crandall1992} since there is no classical solution for such equations. For a convex Hamiltonian $H$ and some regularity on $f$, the viscosity solution is given by Hopf-Lax-Oleinik (HLO) formula \cite{Lions1982,Bardi1984}: 
\begin{align}
\label{hlo}
w(x,t) = \inf_{y \in \R^n } \left\lbrace f(y) + t L\left( \dfrac{x - y}{t} \right)\right\rbrace,
\end{align}
where $L$ is the Lagrangian, defined as the Legendre-Fenchel transform of $H$. Many studies have been proposed on Hopf-Lax-Oleinik viscosity solutions in $\R^n$ \cite{Evans1984,Barles1984,Crandall1986}, and the subject is still of high interests with active research using for example Heisenberg groups \cite{Manfredi2002}, Carnot groups \cite{Balogh2013}, Riemannian manifolds \cite{Fathi2008,Angulo2014,Azagra2015,Diop2021}, Caputo time-fractional derivatives \cite{Camilli2019} or linking the intrinsic HLO semigroup and the intrinsic slope \cite{Donato2023}.

\section{Equivariance and homogeneous spaces on Riemannian manifolds}\label{sect3}
 
Let $ M$ be a smooth manifold and $x \in  M$. A linear mapping $v : C^\infty( M;\mathbb{R}) \rightarrow \mathbb{R}$ satisfying the Leibniz rule:
\begin{equation}
\forall~f_1, f_2 \in C^\infty( M;\mathbb{R}) \quad v(f_1f_2) = f_1(x)v(f_2) + v(f_1)f_2(x) \label{eq:c3eq1}
\end{equation}
is called a derivation at $x$. For all $x\in  M$, the set of derivations at $x$ forms a real vector space of dimension $d$ denoted $T_x M$ so called the tangent space at $x$; its elements can be also called tangent vectors. In Euclidean space, an operator satisfying \eqref{eq:c3eq1} is the derivative along a specific direction, and this definition is a generalization of derivatives on smooth manifolds in general.

Let $G$ be a connected Lie group. We assume that the group \(G\) acts regularly on the spaces \(P\) and \(Q\), meaning that there exists regular maps \(\rho_P : G \times P \rightarrow P\) and \(\rho_Q : G \times Q \rightarrow Q\) respectively defined for all \(r, h \in G\), by:
\begin{align*}
 \rho_P(rh, x) = \rho_P(r, \rho_P(h, x))
 \end{align*}
and
\begin{align*}
\rho_Q(rh, x) = \rho_Q(r, \rho_Q(h, x)),
\end{align*}
making \(\rho_P\) and \(\rho_Q\) group actions on their respective spaces. In addition, we assume that the group \(G\) acts transitively on the spaces (smooth manifolds), meaning that for any two elements in these spaces, there exists a transformation in \(G\) that maps them to each other. This implies that \(P\) and \(Q\) can be viewed as homogeneous spaces.
\begin{definition}
A Riemannian metric on a differentiable manifold $M$ is given by a scalar product $\mu$ on each tangent space $T_xM$ depending smoothly on the base point $x\in M$, that is, $\forall~x\in M$,\\$\mu_x:T_xM\times T_xM\to \mathbb R$ is a symmetric, bilinear and positive definite map, and $\mu_x$ varies smoothly over $M$.\\
A Riemannian manifold $(M,\mu)$ is a differentiable manifold $M$ equipped with a Riemannian metric $\mu$.
\end{definition}

\begin{definition}
Let $G$ a connected Lie group with neutral element $e$ and $(M,\mu)$ a connected Riemannian manifold. A left action of $G$ on $(M,\mu)$ is an application $\varphi:G\times (M,\mu)\to (M,\mu)$ satisfying:
\begin{enumerate}
\item $\varphi(e,x)=x$, $\forall~x\in (M,\mu)$.
\item $\varphi(g,\varphi(h,x)=\varphi(gh, x)$, $\forall~g,h\in G$ and $\forall~x\in (M,\mu)$.
\end{enumerate}
\end{definition}
Let $\varphi:G\times (M,\mu)\to (M,\mu)$ be a left action of $G$ on $(M,\mu)$. For a fixed $g\in G$, we define $\varphi_g:(M,\mu)\to (M,\mu)$; $x\mapsto \varphi_g(x)=\varphi(g,x)$.
We say $\varphi: G\times (M,\mu)\to (M,\mu)$ is a left action if we have
\begin{equation}
\varphi_e=id_M \ \text{and} \ \varphi_g\circ\varphi_h=\varphi_{gh}, ~\forall~ g,h\in G.
\end{equation}
Let $\varphi_h: (M,\mu) \longrightarrow  (M,\mu)$ be the left group action (considered here as a multiplication) by an element \(h \in G\) defined $\forall~x \in  (M,\mu)$ by:
\begin{equation}
\varphi_h(x) = h\cdot x.\label{eq:c3eq2}
\end{equation} 
Let \(\mathcal{L}_h\) be the left regular representation of \(G\) on functions \(f\) defined on \( M\) by:
\begin{equation}
(\mathcal{L}_h f)(x) = f(\varphi_{h^{-1}} (x)),\label{eq:c3eq3}
\end{equation}
with $h^{-1}$ being the inverse of $h\in G$.\\

We consider a layer in a neural network as an operator (from functions on \( M_1\) to functions on \( M_2\)). To ensure the equivarianc of the network, we shall require the operator to be equivariant with respect to the actions on the function spaces.

Let \(x_0\) be an arbitrary fixed point on the connected Riemannian manifold $( M, \mu)$. Let $\pi : G \rightarrow  (M,\mu)$ be the projection defined by assigning to each element \(h\) of \(G\) an element of $(M,\mu)$ in the following:
\begin{equation}
\forall~h\in G\,~~ \pi(h) = \varphi_h(x_0). \label{eq:c3eq4}
\end{equation}
In other words, once a reference point $x_0 \in ( M,\mu)$ is chosen, the projection \(\pi(h)\) assigns to every element \(h\) in \(G\) the unique point in $( M,\mu)$ to which \(h\) sends the chosen reference point \(x_0\) under the action of \(L\) given by \eqref{eq:c3eq2}.

In this work, we consider a connected Lie group \(G\) that acts transitively on the connected Riemannian manifold $( M,\mu)$. This means that for any points $x$ and $y \in ( M,\mu)$, there exists an element \(h \in G\) such that $\varphi_h(x)= y$, corresponding to the definition of an homogeneous space under the action of the group \(G\).
\begin{definition}
Let \(G\) be a connected Lie  group with homogeneous spaces \( M\) and \( N\). Let \(\phi\) be an operator on functions from \( M\) to functions on \( N\). We say that \(\phi\) is equivariant with respect to \(G\) if for all functions \(f\), one has:
\begin{equation}
\forall~h \in G,~ (\phi \circ \mathcal{L}_h) f = (\mathcal{L}_h \circ \phi) f ,\label{eq:c3eq5}
\end{equation}
\end{definition}
We deal here with operators acting on vector and tensor fields; then, making them equivariant will make the process equivariant.\\

Let \(h \in G\), $x\in  (M,\mu)$ and \(T_x M\) be the tangent space of $(M,\mu)$ at the point \(x\). The pushforward of the group action \(\varphi_h\) denoted $(\varphi_h)_*$ is defined by: \((\varphi_h)_* : T_x M \rightarrow T_{\varphi_h(x)} M\) such that for all smooth functions \(f\) on $( M,\mu)$ and all \(v \in T_x M\), one has:
\begin{equation}
((\varphi_h)_* v ) f := v (f \circ (\varphi_h)_*) . \label{eq:c3eq6}
\end{equation}

For all $x \in  (M,\mu)$, we refer to \(G\)-invariance of vector fields \(X : x \mapsto T_x M\) if $\forall~h \in G$ and for all differentiable functions \(f\), one has:
\begin{equation}
X(x)f = X(\varphi_h(x))[\mathcal{L}_h f]. \label{eq:c3eq7}
\end{equation}

\begin{definition} A vector field \(X\) on $( M,\mu)$ is invariant with respect to \(G\) if $\forall~h~\in G$ and $\forall~x \in  (M,\mu)$, one has:
\begin{equation}
X(\varphi_h(x)) = (\varphi_h)_*X(x).\label{eq:c3eq8}
\end{equation} 
\end{definition}


\begin{definition} \label{inv_met} A \((0, 2)\)-tensor field \(\mu\) on \( M\) is \(G\)-invariant if $\forall~h \in G$, $\forall~x \in  M$ and $\forall~v, w \in T_x( M)$, one has:
\begin{equation}
\mu|_{h}(v,w) = \mu|_{\varphi_h(x)}((\varphi_h)_* v, (\varphi_x)_* w). \label{eq:c3eq9}
\end{equation}
\end{definition}
It follows from Definition~\ref{inv_met} that properties derived from metric tensor field \(G\) invariance and vector field \(G\) invariance are the same.
\begin{definition}
\label{dist}
Let $(M,\mu)$ a connected Riemannian manifold, $x, y \in  (M,\mu)$. The distance between $x$ and $y$ is defined as follows:
\begin{align}
d_{\mu}(x, y) = \inf_{{\gamma ~\in ~\Gamma_{t}(x,y)}} \int_{0}^{t}\sqrt{\mu|_{\gamma(t)}(\dot{\gamma}(t),\dot{\gamma}(t))}dt,
\end{align}
with $\Gamma_{t}(x,y) =\{\gamma :[0,t]\longrightarrow  (M,\mu)~~ \text{of class}~~ C^1 ,\gamma (0)= x ~~\text{and}~~ \gamma(t)=y\}$.
\end{definition}

\begin{definition} The cut locus is defined as the set of points $x\in M$ (or $h\in G$) from which the distance map is not smooth (except at $x$ or $h$).
\end{definition}

\begin{proposition} Let $x,y \in  (M,\mu)$ such that $\varphi_h(y)$ is away from the cut locus of $\varphi_h(x)$. Then, $\forall~h \in G$, one has:
\label{inv_geod}
\begin{equation}
d_\mu(x,y) = d_\mu\big(\varphi_h(x), \varphi_h(y)\big). \label{eq:eq10}
\end{equation}

\end{proposition}
\begin{proof}
Let us perform a left multiplication by $h$ in one direction and by $h^{-1}$ in the other direction. A bijection can then be established between $C^1$ curves connecting $x$ to $y$ and connecting $\varphi_h(x)$ to $\varphi_h(y)$. Thus, we have:
\begin{align*}
d_\mu\big( \varphi_h(x),\varphi_h(y)\big) &= \inf_{{\beta ~\in ~\Gamma_{t}(\varphi_h(x), \varphi_h(y))}}\int_{0}^{t}\sqrt{\mu|_{\beta(t)}(\dot{\beta}(t),\dot{\beta}(t))}dt,\\
&= \inf_{{h\gamma ~\in ~\Gamma_{t}(\varphi_h(x),\varphi_h(y))}}\int_{0}^{t}\sqrt{\mu|_{h\gamma(t)}\Big(\varphi_h(\dot{\gamma}(t)), \varphi_h(\dot{\gamma}(t))\Big)}dt ~~~~~\\\text{with} ~~ \gamma \in \Gamma_{t}(\varphi_h(x),\varphi_h(y))\\
&= \inf_{{h\gamma ~\in ~\Gamma_{t}(\varphi_h(x),\varphi_h(y))}}\int_{0}^{t}\sqrt{\mu|_{h\gamma(t)}\Big((\varphi_h)_{*}\dot{\gamma}(t),(\varphi_h)_{*}\dot{\gamma}(t)\Big)}dt\\
&= \inf_{{h\gamma ~\in ~\Gamma_{t}(\varphi_h(x),\varphi_h(y))}}\int_{0}^{t}\sqrt{\mu|_{\gamma(t)}(\dot{\gamma}(t),\dot{\gamma}(t))}dt ~~~~~~~~\text{by \eqref{eq:c3eq9}}\\
&= \inf_{{\gamma ~\in ~\Gamma_{t}(x,y)}}\int_{0}^{t}\sqrt{\mu|_{\gamma(t)}\big(\dot{\gamma}(t),\dot{\gamma}(t)\big)}dt = d_\mu(x,y)  \qed
\end{align*}
\end{proof}
\begin{remark}
Staying away from the cut locus provides a unique distance in Definition~\ref{dist}. Also, thanks to Proposition~\ref{inv_geod}, $d_\mu$ shares the same symmetries, since we derive it from a tensor field invariant under $G$.
\end{remark}

\section{Group morphological convolutions and PDEs}\label{sect4}
Let $(M,\mu)$ be a compact and connected Riemannian manifold endowed with a metric $\mu$, and $f,b: (M,\mu)\longrightarrow \mathbb{R}$.
\begin{definition}
\label{def:conv_gp}
The group morphological convolution $\diamondsuit$ between $b$ and $f$ is defined $\forall~x \in  (M,\mu)$ by: $\displaystyle{b\diamondsuit f(x) = \inf_{p\in G}\{f(\varphi_p(x_0))+b(\varphi_{p^{-1}}(x))\}}$.
\end{definition}
Denote $TM$ the tangent bundle $(M,\mu)$ and $L:TM\to\mathbb{R}$ a Lagrangian function. 
\begin{equation}
h_t(x,y)=\inf_{\substack{\gamma}}\Big\{\int_0^t L(\gamma(s),\gamma^\prime(s))\ud s;  \gamma:[0,t]\to (M,\mu)  \text{ of class } C^1,
\gamma(0)= x,\gamma(t)=y\Big\}.
\end{equation}
\begin{definition}
A function $u:V\to\mathbb{R}$ is a viscosity subsolution of $H(x, d_xu)=c\in \mathbb{R}$ on the open subset $V\subset (M,\mu)$, where $d_xu$ is the differential of $u$ at a point $x\in (M,\mu)$, if for every $\mathcal C^1$ function $\varphi:V\to\mathbb{R}$ with $\varphi\geq u$ everywhere, and at every point $x_0\in V$ where $u(x_0)=\varphi(x_0)$, one has $H(x_0,d_{x_0}\varphi)\leq c$.\\
A function $u:V\to\mathbb{R}$ is a viscosity supersolution of $H(x,d_xu)=c$ on the open subset $V\subset (M,\mu)$, if for every $\mathcal C^1$ function $\varphi:V\to\mathbb{R}$, with $u\geq\varphi$ everywhere, and at every point $y_0\in V$ where $u(y_0)=\varphi(y_0)$, one has: $H(y_0,d_{y_0}\varphi)\geq c$.\\
A function $u:V\to\mathbb{R}$ is a viscosity solution of $H(x,d_xu)=c$ on the open subset $V\subset (M,\mu)$, if it is both a viscosity subsolution and a viscosity supersolution.
\end{definition}
Let $H:T^\ast\!M\to\mathbb R$ be the Hamiltonian associated to the Lagrangian $L$, $H$ is defined on the cotangent bundle $T^\ast\!M$ of $(M,\mu)$, $H(x,q)=\underset{v\in T_x\!M}\sup\{q(v)-L(x,v)\}$. The first-order Hamilton-Jacobi PDE \eqref{pde_hamil} can be extended in Riemannian manifolds as follows: $$\partial_t w + H\left(x,\nabla w\right) = 0 \text{ in } (M,\mu) \times (0,+\infty);~w(\cdot,0) = f \text{ on } (M,\mu).$$ 
\begin{definition}
$L$ is a Tonelli Lagrangian if the above conditions are fulfilled:
\begin{enumerate}
\item $L: TM\to\mathbb{R}$ is of class $C^2$, at least.
\item $L$ is superlinear above compact subset of $M$; \ie $\underset{\|p\|\to\infty}\lim\dfrac{L(p)}{\|p\|}= +\infty$, $\|\cdotp\|$ being a norm induced by a Riemannian metric on $M$.
\item For each $(x,v)\in TM$, $\dfrac{\partial^2L}{\partial^2v}(x,v)$ is positive definite as a quadratic form.
\end{enumerate}
\end{definition}

\begin{theorem}[\cite{Fathi2008}]
\label{th:c3th11}
Let $L : TM\rightarrow \mathbb{R}$ be a Tonelli Lagrangian. If $f \in C^0(M,\R)$, then the function $w : M\times[0;+\infty]\rightarrow \R$ defined by:
\begin{equation}
w(x,t) = \inf_{y\in M}\left\{f(y) + h_t(x, y)\right\}
\end{equation}
is a viscosity solution of the equation:
\begin{equation}
\frac{\partial w}{\partial t} + H(x, \nabla w) = 0 \text{ in } M\times (0,+\infty); w(\cdotp,0)=f \text{ on } M,
\end{equation}
with $H$ being the Hamiltonian associated with $L$.
\end{theorem}
By reversing the time, the viscosity solution of the PDE:
\begin{equation}
\frac{\partial w}{\partial t} - H(x, \nabla w) = 0 \text{ in } M\times (0,+\infty); w(\cdotp,0)=f \text{ on } M,
\end{equation}
is given by: 
\begin{equation}
w(x,t) = \sup_{y\in M}\left\{f(y) - h_t(x, y)\right\}.
\end{equation}
Riemannian multiscale operations can be performed by choosing a specific Hamiltonian, respectively, $H=\lVert\nabla_{\mu} w\rVert^k_{\mu}$ for the multiscale dilations and \\$H=-\lVert\nabla_{\mu} w\rVert^k_{\mu}$ for multiscale erosions. Doing so links mathematical morphology to first order Hamilton-Jacobi PDEs, and taking $k>1$ allows to deal with more general structuring functions than the quadratic ones. 

\begin{proposition}
\label{prop:conv}
Let $f \in C^0( (M,\mu), \mathbb{R})$ a continuous function and let \\$c_k = \frac{k-1}{k^{\frac{k}{k-1}}}$, $k>1$. Viscosity solutions of the Cauchy problem
\begin{equation}
\label{eq:c3eq28}
\dfrac{\partial w}{\partial t} + \|\nabla_{\mu} w\|^k_{\mu} = 0 \text{ in }  (M,\mu) \times (0;~\infty);~w(\cdotp,0) = f \text{ on } (M,\mu),
\end{equation} 
are given by: $\displaystyle{f_t(x) = b_{t}^{k} \diamondsuit f(x) := \inf_{h \in G} \left\{ f\big(\varphi_h(x_0)\big) + c_k  \frac{d_\mu\big(\varphi_{h^{-1}}(x),x_0\big)^\frac{k}{k-1}}{t^{\frac{1}{k-1}}} \right\}}$, \\where $b^{k}_{t}=c_k \dfrac{d_\mu(x_0,\cdotp)^\frac{k}{k-1}}{t^{\frac{1}{k-1}}}$ are the multiscale structuring functions. 
\end{proposition}
\begin{proof}
Viscosity solutions of the PDE \eqref{eq:c3eq28} are given by HLO formulas \cite{Diop2021}: $$\displaystyle{
f_t(x)= \inf_{y \in  M} \left\{ f(y) + c_k  \frac{d_\mu(x,y)^\frac{k}{k-1}}{t^{\frac{1}{k-1}}} \right\}}.$$ The projection $\pi$ \eqref{eq:c3eq4} is defined by associating any $h\in G$ to an element $x \in  (M,\mu)$. Then, using the definition and accounting the invariance property in Proposition~\ref{inv_geod}, one gets:
\begin{align*}
f_t(x) &= \inf_{h\in G} \left\{ f\big(\varphi_h(x_0)\big) + c_k\dfrac{d_\mu(x, \varphi_h(x_0))^\frac{k}{k-1}}{t^\frac{1}{k-1}}\right\}\\
&= \inf_{h\in G} \left\{ f\big(\varphi_h(x_0)\big) + c_k\frac{d_\mu\big(\varphi_{h^{-1}}(x), x_0\big)^\frac{k}{k-1}}{t^\frac{1}{k-1}}\right\} \\
&=\inf_{h\in G}\left\{ f\big(\varphi_h(x_0)\big)+ b_{t}^k\big(\varphi_{h^{-1}}(x)\big)\right\} = b_{t}^k \diamondsuit f(x). \qed
\end{align*}
\end{proof}
By reversing the time, we can prove that the viscosity solutions of the Cauchy problem corresponding to multiscale dilations:
\begin{equation}
\dfrac{\partial w}{\partial t} - \|\nabla_{\mu} w\|^k_{\mu} = 0 \text{ in }  (M,\mu) \times (0;~\infty);~w(\cdotp,0) = f \text{ on } (M,\mu)
\end{equation} 
are given by \cite{Diop2021}: $$\displaystyle{f^t(x) = \sup_{x \in  (M,\mu)} \left\{ f(y) - C_k  \frac{d_\mu(x, y)^\frac{k}{k-1}}{t^{\frac{1}{k-1}}} \right\}},$$ and thus, using the same arguments as in the preceding proof, one has: $$f_t(x) = -(b_{t}^{k} \diamondsuit (-f))(x).$$
\begin{proposition}
Let $k>1$. Then, $\forall~t,s \geq 0$, the family of structuring functions $b_{t}^{k}$ satisfy the following semigroup property: $
b_{t+s}^{k} = b_{t}^{k} \diamondsuit  b_{s}^{k}$.
\end{proposition}

\begin{proof}
Indeed, one has: 
\[
b_{t+s}^{k}\diamondsuit f(x)= \inf_{p\in G}\{f(\varphi_p(x_0))+b_{t+s}^{k}(\varphi_{p^{-1}}(x))\}.
\]
Then, using  Theorem~2.1-(ii) in \cite{Balogh2012}, one gets:
\begin{align*}
b_{t+s}^{k}\diamondsuit f(x) &= \inf_{h\in G}\{b_{s}^{k}\diamondsuit f(\varphi_h(x_0))+b_{t}^{k}(\varphi_{h^{-1}}(x)) \} \\
&= (b_{t}^{k}\diamondsuit b_{s}^{k})\diamondsuit f(x). \qed
\end{align*} 
\end{proof}

\section{Morphological equivariant PDEs for generative models}\label{sect5}

We aim at proposing generative models for images that are based on PDEs satisfying an equivariance property. Our approach is resumed in two major steps: {\bf i)} designing morphological PDEs as an alternative for traditional CNNs that preserve an equivariant processing in composing feature maps in layers, and {\bf ii)} proposing a generative model based on this structure.

\subsection{Morphological PDE-based layers}
Feature maps are carried out in traditional CNNs throughout the classical convolution, pooling and ReLU activation functions. Our goal is to propose PDEs that behave like traditional CNNs, in one hand, and preserve an equivariance property, on the other hand. For that purpose, PDEs will be formulated on group transformations to ensure equivariance and make PDEs consistent with G-CNNs \cite{Cohen2016,Bekkers2018,Cohen2019}. Equivariance is a robust way to incorporate desired and essential symmetries into the network so that there is no more need to learn such symmetries; consequently, the amount of data is reduced. Viewing layers as image processing operators allows us use well elaborated image analysis and processing techniques to design the network. Thin image analysis is needed to achieve our objective. Due to its nonlinearity aspects, good shape and geometry description capabilities, mathematical morphology appeared as an efficient and powerful tool for multiscale image and data analysis \cite{Shih2009}. For a better analysis of geometrical image structures, it is also interesting to consider works from geometric image analysis \cite{Welk2019a,Fadili2015a,DubrovinaKarni2014,Burger2016,Duits2007a}. Image and data analysis and processing methods based on non-Euclidean metrics; for instance, Riemannian metrics, are well known to improve a lot Euclidean based approaches. Riemannian manifolds are proved to behave very well for capturing thin data structures, providing then better representations and analysis of geometrical structures present in the data. This fact is shown in many image processing studies with real life applications; for instance, in video surveillance, shape and surface analysis, human body and face analysis, image segmentation \cite{Su2014,Balan2015,Citti2016,Kurtek2016,Younes2019,Pierson2022a}. For these reasons, we choose homogeneous spaces to avoid Euclidean metrics so that the network is provided with image processing capabilities for a better handling of geometric thin structures \cite{Citti2006,Janssen2018,Franceschiello2019,Duits2019,Diop2021}. Doing so should make feature maps richer, and combined with the equivariance property of the morphological PDEs will provide neat improvements of classical GANs in terms of quality of the content generation. Morphological PDEs are thus used to replace the pooling operations and ReLU activation functions in the proposed generative model.

\subsection{PDE model design}

Let $(M,\mu)$ be a compact and connected Riemannian manifold, $f: (M,\mu)\longrightarrow \mathbb{R}$.\\

PDE-G-CNNs were formally introduced in homogeneous spaces with $G$-invariance metric tensor fields on quotient spaces \cite{Smets13July2022}. Built on the primary approach, the proposed model is based on a combination of traditional CNNs and morphological PDE layers of Hamilton-Jacobi type in Riemannian manifolds, and is composed of the following PDEs: \\
$\bullet $ Convection: $$\dfrac{\partial w}{\partial t} + \alpha w = 0 \text{ in } (\mathcal{M},\mu)\times (0,~\infty);~w(\cdot, 0) = f \text{  on } (\mathcal{M},\mu).$$
$\bullet $ Diffusion: $$\dfrac{\partial w}{\partial t} + (-\Delta_\mu)^{k/2} w = 0 \text{ in } (\mathcal{M},\mu)\times (0,~\infty);~ w(\cdot, 0) = f \text{ on } (\mathcal{M},\mu).$$
$\bullet $ Morphological multiscale erosions and dilations for ($+$) and ($+$) sign:
\begin{equation}
\label{ero_dil}
\dfrac{\partial w}{\partial t} \pm \|\nabla_\mu w\|^{{k}}_\mu = 0 \text{ in } (\mathcal{M},\mu)\times (0,~\infty);~w(\cdot, 0) = f \text{ on } (\mathcal{M},\mu),
\end{equation} 
where \(\alpha\) a is vector field invariant under \(G\) on $(\mathcal{M},\mu)$, \(\Delta_\mu\) represents the Laplace-Beltrami operator, \(\Vert \cdot \Vert_\mu\) the norm induced by the Riemannian metric $\mu$ and \(k>1\). The above system of PDEs consitutes the PDE model solved in a step basis using the operator splitting method, where each step corresponds to one of the PDEs. In this work, we only use the morphological multiscale operations steps \eqref{ero_dil}, the convection and diffusion terms are left for future work. Thus, our PDE layers are defined by the following PDEs:
\begin{equation*}
     \label{eq:c3eq13}
\begin{cases}
        \dfrac{\partial w}{\partial t} \pm \|\nabla_{\mu} w\|^{{k}}_{\mu} = 0  & \text{ on } (\mathcal{M},\mu)\times (0,\infty)\\
         w(\cdot, 0) = f & \text{ in } (\mathcal{M},\mu).
\end{cases}
\end{equation*}
PDEs \eqref{ero_dil} introduce nonlinearities into the generator network of the GM-GAN using morphological convolutions, which are obtained a viscosity sense and given respectively for multiscale dilations and erosions thanks to Proposition~\ref{prop:conv}, by: $$f_t(x) = b_{t}^{k} \diamondsuit f(x) \text{ and } f^t(x) = -(b_{t}^{k} \diamondsuit (-f))(x),$$ where $b^{k}_{t}=c_k \dfrac{d_\mu(x_0,\cdot)^{\frac{k}{k-1}}}{t^{\frac{1}{k-1}}}$.
Next, we show that layers introduce nonlinearities in traditional CNNs, max pooling and ReLUs can be seen as morphological convolutions:
\begin{proposition}
Let $f \in C^{\infty}((\mathcal{M},\mu))$ and $B \subset (\mathcal{M},\mu)$ an non-empty set. Consider the flat structuring function $b: (\mathcal{M},\mu) \rightarrow \mathbb{R} \cup\{\infty\}$. Then, one has: $$\displaystyle{-\left(b \diamondsuit (-f)\right)(x)=\underset{\varphi_{h^{-1}}(x)\in B}{\sup _{h \in G}} f\left(\varphi_{h}(x_0)\right)}.$$
\end{proposition}

\begin{proof} Using the definition of group convolutions \ref{def:conv_gp}, one gets:
\[
\begin{aligned}
 -\left(b \diamondsuit (-f)\right)(x) 
& =-\inf \left\{ \underset{\varphi_{h^{-1}}(x) \in B}{\inf _{h \in G}}-f\left(\varphi_{h}(x_0)\right), \underset{\varphi_{h^{-1}}(x) \notin B}{\inf _{h \in G}}-f\left(\varphi_{h}(x_0)\right)+\infty  \right\} \\
& =-\underset{\varphi_{h^{-1}}(x) \in B}{\inf _{h \in G}}-f\left(\varphi_{h}(x_0)\right) \\
& =\underset{\varphi_{h^{-1}}(x) \in B}{\sup _{h \in G}} f\left(\varphi_{h}(x_0)\right). \qed
\end{aligned}
\]
\end{proof}
The max pooling of function $f$ with motif $B$ can in fact be seen as a flat morphological dilation with a structurant element $B$. It is truly the case for example for $\R^n$. Indeed, for $f \in C^{0}\left(\mathbb{R}^{n}\right)$ and $B \subset \mathbb{R}^{n}$ a compact set, for every $x \in \mathbb{R}^{n}$, one has: $$\displaystyle{-\left(b \diamondsuit_{\mathbb{R}^{n}}(-f)\right)(x)=\sup _{y \in B} f(x-y)},$$ where the right hand side of the preceding equation is in fact a flat dilation with a structurant element $B$ (see \eqref{inf_sup:st} in Definition~\ref{def:morpho}). 
\begin{proposition}
Let $f\in C_c^0((\mathcal{M},\mu))$. Morphological dilation with the following structuring function: $b(x)= 0$, if $x=x_{0}$; and $b(x)=\underset{x \in \mathcal{M}}{\sup} f(x)$, otherwise, is exactly the same as applying a ReLU to $f$: $$\displaystyle{-\left(b \diamondsuit (-f)\right)(x)=\max \{0, f(x)\}}.$$
\end{proposition}

\begin{proof}
Using $b$ in the definition of morphological group convolution yields:
$$
\begin{aligned}
-\left(b \diamondsuit -f\right)(x)
& =-\inf _{h \in G}\left\{ b\left(\varphi_{h^{-1}}(x)\right)-f\left(\varphi_{h}(x_0)\right)\right\} \\
& =-\inf _{h \in G}\left\{\inf _{\varphi_{h^{-1}}(x)=x_{0}}-f\left(\varphi_{h}(x_0)\right), \inf _{\varphi_{h^{-1}}(x) \neq x_{0}}-f\left(\varphi_{h}(x_0)\right)+\sup _{y \in \mathcal{M}} f(y)\right\} \\
& =\sup \left\{f(x), \underset{z \neq x}{\sup _{z \in \mathcal{M}}} f(z)-\sup _{y \in \mathcal{M}} f(y)\right\},
\end{aligned}
$$
The existence of the supremum of $f$ is guaranteed since $f$ is continuous on a compact support; moreover, one has $\displaystyle{\underset{ z \neq x_{0}}{\sup _{z \in \mathcal{M}}} f(z)=\sup _{y \in \mathcal{M}} f(y)}$. Thus, one gets:
$$
-\left(b \diamondsuit (-f)\right)(x)=\max \{f(x), 0\}. \qed
$$
\end{proof}

\subsection{Architecture of morphological equivariant PDEs based on GAN}
 
We present here a generative model based on morphological equivariant convolutions in PDE-G-CNNs in order to provide nonlinearity in classical CNNs in GANs. We choose to work with GANs due to their simplicity and performance. As shown in the preceding section, morphological convolutions allow to introduce equivariant nonlinearities with respect to other transformations, which should turn out to improve the capacity to better capture data information.

Similarly to GAN, the proposed geometric morphological GAN (GM-GAN) is composed of two networks: a generator (G) and a discriminator (D) which are both multi-layer perceptrons. As detailed in the preceding section, we introduce into the network morphological PDE-based layers through the resolution in a step basis of Hamilton-Jacobi PDEs \eqref{ero_dil}, whose viscosity solutions are given for multiscale erosions and dilations thanks to Proposition~\ref{prop:conv}, as:
\begin{align*}
&\displaystyle{f_t(x) = \inf_{h\in G} \left\{ f\big(\varphi_h(x_0)\big) + c_k\frac{d_\mu\big(\varphi_{h^{-1}}(x), x_0\big)^\frac{k}{k-1}}{t^\frac{1}{k-1}}\right\}} \text{ and }\\
&\displaystyle{f^t(x) = \sup_{h\in G} \left\{ f\big(\varphi_h(x_0)\big) - c_k\frac{d_\mu\big(\varphi_{h^{-1}}(x), x_0\big)^\frac{k}{k-1}}{t^\frac{1}{k-1}}\right\}}. 
\end{align*}
For computation purpose, we provide the distance $d_\mu$ in the geodesic ball by considering the hyperbolic ball: $$B = \lbrace (x_1,x_2) \in \mathbb{R}^2 \text{ such that } x_1^2 + x_2^2 < 1 \rbrace,$$ which is endowed with the metric: $$\mu = \dfrac{4(\ud x_1^2 + \ud x_2^2)}{(1 - \lVert x \rVert^2)^2},$$ where \(\lVert \cdot \rVert\) denotes the Euclidean norm in $\R^2$. The distance is obtained as follows: $$\displaystyle{d_\mu(x, y) = \text{Argcosh} \left(1 + \frac{2\lVert x - y \rVert^2}{(1 - \lVert x \rVert^2)(1 - \lVert y \rVert^2)}\right)}.$$
Concave structuring functions  $b^{k}_{t}=c_k \dfrac{d_\mu(x_0,\cdotp)^\frac{k}{k-1}}{t^{\frac{1}{k-1}}}$ are represented in Fig.~\ref{struct_fct} for different values of $t$ and $k$ in $]-1;~1[$.\\
\begin{figure}[htbp] 
    \centering
     \begin{tabular}{cccc}
    \includegraphics[scale=0.4]{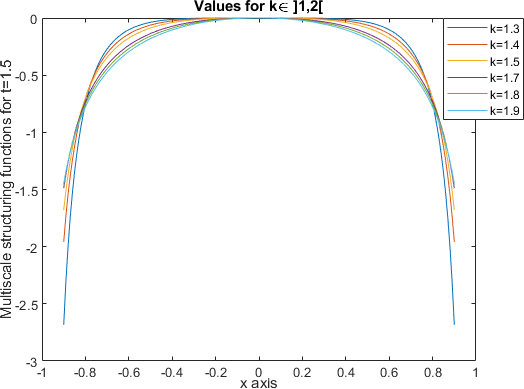} 
    &
    \includegraphics[scale=0.4]{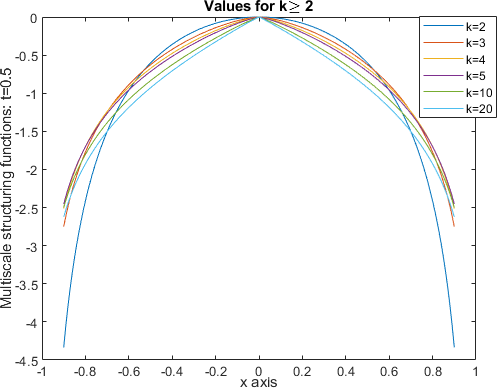} \\
    (a)
    &
    (b)
    \end{tabular}
    \caption{$b_t^k(x)$, $x\in ]-1;~1[$: (a) for $t=1.5$ and $k\in ]1;~2[$. (b) for $t=0.5$ and $k\geq 2$.}
 \label{struct_fct}
\end{figure}

GM-GAN training procedure remains the same as traditional GANs. Specifically, the training procedure is carried out separately but simultaneously. The model takes as input some noise $z$ defined with a prior probability $p_z$, and then, attempts to learn the distribution of the generator $p_g$, by representing a function $G(z; \theta_g)$ from $z$ to the data space. The discriminator network $D$ takes an input image $x$ and finds a function $D(x; \theta_d)$ from $x$ to a single scalar, which is the probability that the image $x$ comes from $p_{data}$ which defines the origin of the sampled images. The output of the $D$ network returns a value close to $1$ if $x$ is a real image from $p_{data}$, and a value very close to $0$ if $x$ comes from $p_g$; otherwise. The main objective of network $D$ is to maximize $D(x)$ for an image coming from the true data distribution $p_{data}$, while minimizing $D(x) = D(G(z; \theta_g))$ for images generated from $p_z$ and not from $p_{data}$. The objective of the generator $G$ is to deceive the $D$ network, meaning to maximize $D(G(z; \theta_g))$. This is equivalent to minimize $1 - D(G(z; \theta_g))$ as $D$ is a binary classifier. This conflict between these objectives is called the minimax game and formulated as follows: $$\displaystyle{\min \max E_{x\sim p_{data}(x)}[\log D(x)] + E_{z\sim p_z(z)}[\log(1 - D(G(z; \theta_g)))]}.$$ The case $p_g = p_{data}$ corresponds to the global optimum of the minimax game. Main contributions of the proposed GM-GAN rely on the equivariance property and non linearity characteristics brought out by group morphological convolutions and their ability to extract thin geometrical features, which lead to richer feature maps and a reduction of the amount training data. \\

For the GM-GAN generator, let \(x\) be the input data into the morphological layer called {\it Morphoblock}. Then, \(x\) goes first through a multiscale morphological erosion operation, followed by a multiscale morphological dilation. Afterwards, both erosion and dilation are followed by a linear convolution. The output of the PDE layer is obtained by a linear combination of the two outputs. The overall architecture of the GM-GAN generator is illustrated in Fig.~\ref{fig:f2}.
\begin{figure}[h]
    \centering
    \includegraphics[width=12cm]{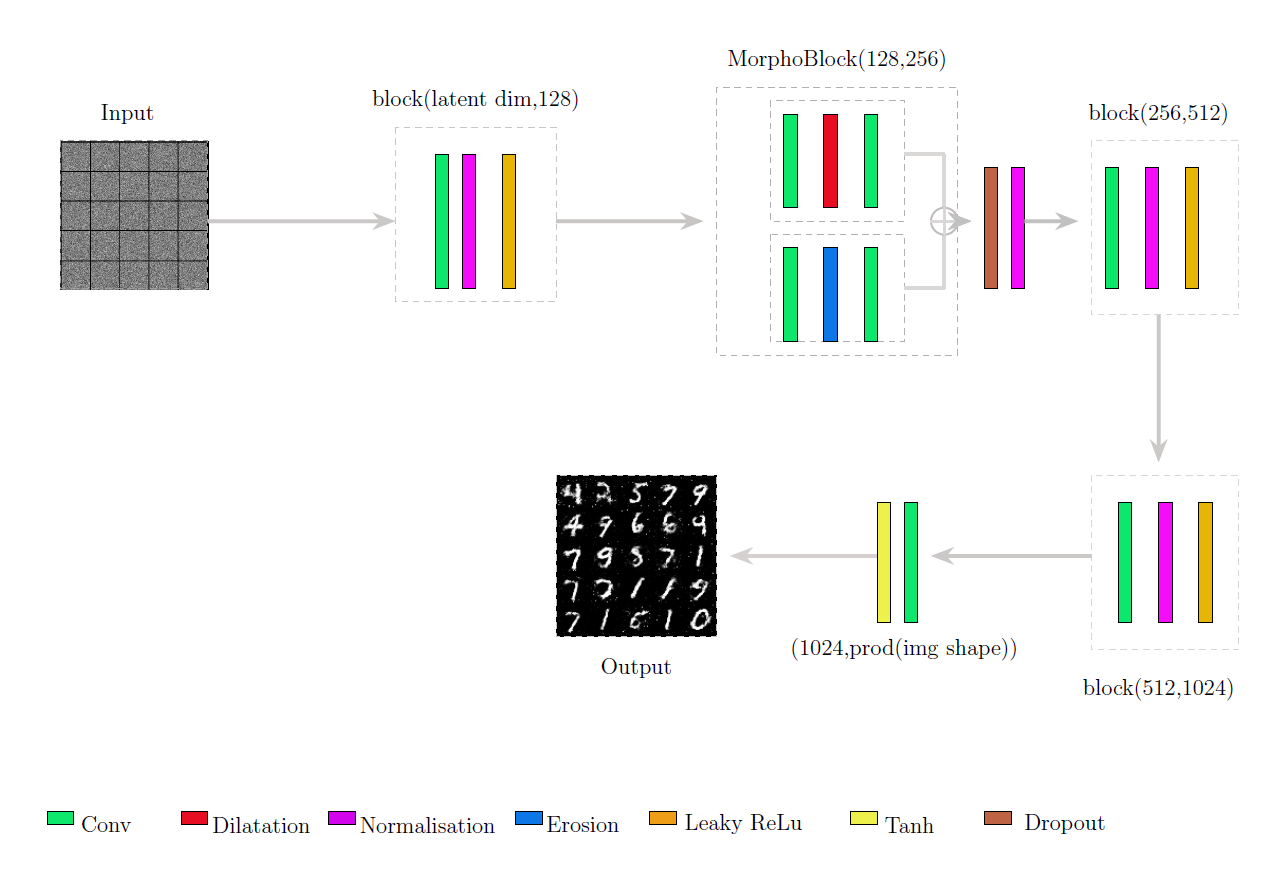}
    \caption{Architecture of GM-GAN generator.}
    \label{fig:f2}
\end{figure}

\section{Numerical experiments}\label{sect6}

GM-GAN and GAN are applied to MNIST dataset. MNIST database consists of $70,000$ black-and-white $28$x$28$ images that represent handwritten digits from $0$ to $9$. It is divided into a training set of $60,000$ images and a test set of $10,000$ images. Same training parameters are set for GM-GAN and GAN: number of epochs to $200$, the batch size to $64$, the latent space dimensionality to $100$, and the interval between image samples to $400$. Generated images with GM-GAN and GAN are displayed in Fig.~\ref{fig:f3} showing higher generation quality with GM-GAN in comparison to traditional GAN. 
\begin{figure}[htbp]
  \centering
  \begin{subfigure}{0.22\textwidth}
    \includegraphics[width=\linewidth]{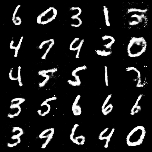}
    \caption{GM-GAN: $75$}
    \label{subfig:m70}
  \end{subfigure}\hspace{1em}%
  \begin{subfigure}{0.22\textwidth}
    \includegraphics[width=\linewidth]{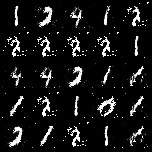}
    \caption{GAN: $75$}
    \label{subfig:gan70}
  \end{subfigure}\hspace{1em}%
  \begin{subfigure}{0.22\textwidth}
    \includegraphics[width=\linewidth]{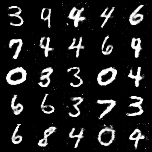}
    \caption{GM-GAN: $100$}
    \label{subfig:m90}
  \end{subfigure}\hspace{1em}%
  \begin{subfigure}{0.22\textwidth}
    \includegraphics[width=\linewidth]{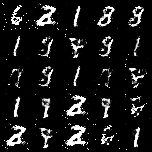}
    \caption{GAN: $195$}
    \label{subfig:gan195}
  \end{subfigure}
  \\ 
  \begin{subfigure}{0.22\textwidth}
    \includegraphics[width=\linewidth]{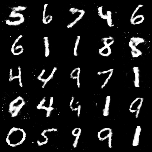}
    \caption{GM-GAN: $80$}
    \label{subfig:m71}
  \end{subfigure}\hspace{1em}%
  \begin{subfigure}{0.22\textwidth}
    \includegraphics[width=\linewidth]{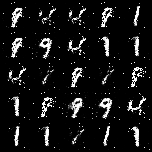}
    \caption{GAN: $80$}
    \label{subfig:gan71}
  \end{subfigure}\hspace{1em}%
  \begin{subfigure}{0.22\textwidth}
    \includegraphics[width=\linewidth]{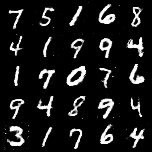}
    \caption{GM-GAN: $105$}
    \label{subfig:m91}
  \end{subfigure}\hspace{1em}%
  \begin{subfigure}{0.22\textwidth}
    \includegraphics[width=\linewidth]{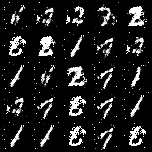}
    \caption{GAN: $196$}
    \label{subfig:gan196}
  \end{subfigure}
  \\
  \begin{subfigure}{0.22\textwidth}
    \includegraphics[width=\linewidth]{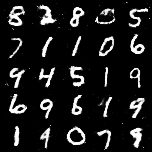}
    \caption{GM-GAN: $85$}
    \label{subfig:m72}
  \end{subfigure}\hspace{1em}%
    \begin{subfigure}{0.22\textwidth}
    \includegraphics[width=\linewidth]{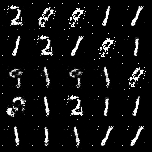}
    \caption{GAN: $85$}
    \label{subfig:gan72}
  \end{subfigure}\hspace{1em}%
      \begin{subfigure}{0.22\textwidth}
    \includegraphics[width=\linewidth]{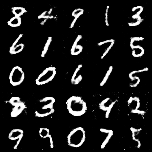}
    \caption{GM-GAN: $110$}
    \label{subfig:m92}
  \end{subfigure}\hspace{1em}%
      \begin{subfigure}{0.22\textwidth}
    \includegraphics[width=\linewidth]{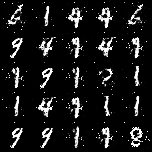}
    \caption{GAN: $197$}
    \label{subfig:gan197}
  \end{subfigure} 
  \\ 
  \begin{subfigure}{0.22\textwidth}
    \includegraphics[width=\linewidth]{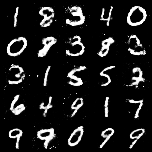}
    \caption{GM-GAN:$90$}
    \label{subfig:m73}
  \end{subfigure}\hspace{1em}%
  \begin{subfigure}{0.22\textwidth}
    \includegraphics[width=\linewidth]{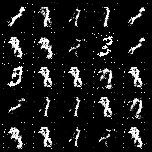}
    \caption{GAN: $90$}
    \label{subfig:gan73}
  \end{subfigure}\hspace{1em}%
  \begin{subfigure}{0.22\textwidth}
    \includegraphics[width=\linewidth]{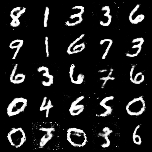}
    \caption{GM-GAN: $115$}
    \label{subfig:m93}
  \end{subfigure}\hspace{1em}%
  \begin{subfigure}{0.22\textwidth}
    \includegraphics[width=\linewidth]{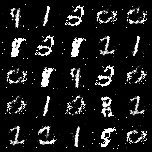}
    \caption{GAN: $198$}
    \label{subfig:gan198}
  \end{subfigure}
  \\ 
  \begin{subfigure}{0.22\textwidth}
    \includegraphics[width=\linewidth]{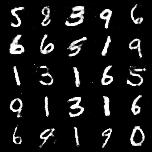}
    \caption{GM-GAN: $95$}
    \label{subfig:m74}
  \end{subfigure}\hspace{1em}%
  \begin{subfigure}{0.22\textwidth}
    \includegraphics[width=\linewidth]{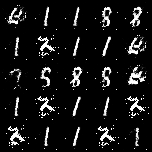}
    \caption{GAN: $95$}
    \label{subfig:gan74}
  \end{subfigure}\hspace{1em}%
  \begin{subfigure}{0.22\textwidth}
    \includegraphics[width=\linewidth]{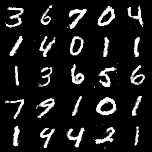}
    \caption{GM-GAN: $120$}
    \label{subfig:m94}
  \end{subfigure}\hspace{1em}%
  \begin{subfigure}{0.22\textwidth}
    \includegraphics[width=\linewidth]{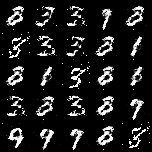}
    \caption{GAN: $199$}
    \label{subfig:gan199}
  \end{subfigure}
  \\ 

  \caption{Image generation using MNIST: GM-GAN vs. GAN.}
  \label{fig:f3}
\end{figure}
This can be seen by comparing images produced at epochs $70$ to $95$ with GM-GAN (Figs.~\ref{subfig:m70}, \ref{subfig:m71}, \ref{subfig:m72}, \ref{subfig:m73} and \ref{subfig:m74}) and those generated with GAN at same epochs (Figs.~\ref{subfig:gan70}, \ref{subfig:gan71}, \ref{subfig:gan72}, \ref{subfig:gan73} and \ref{subfig:gan74}). For instance, some digits are clearly identifiable with GM-GAN based generation, whereas it is almost impossible to recognize the digits with GAN based ones. We also observe that the images generated with GM-GAN at epochs going from $100$ to $120$ (Figs.~\ref{subfig:m90}, \ref{subfig:m91}, \ref{subfig:m92}, \ref{subfig:m93} and \ref{subfig:m94}) are of better quality than generated ones with GAN for the last five epochs going from epoch $195$ to $199$ (Figs.~\ref{subfig:gan195}, \ref{subfig:gan196}, \ref{subfig:gan197}, \ref{subfig:gan198} and \ref{subfig:gan199}). To better discriminate that fact, we zoom in on some areas in images generated at epochs $85$, $92$ and $96$ (Figs.~\ref{fig:f4}-(a)-(b), (c)-(d) and (e)-(f); respectively), and highlight the realistic variations between the generated images of the same digit. This indicates that GM-GAN has a deeper understanding of the sample characteristics and is capable of generalizing them beyond the specific examples they are trained on. This can be observed in Fig.~\ref{fig:f4}-(b) with digits $3$ and $6$, in Fig.~\ref{fig:f4}-(d) with digits $2$ and $8$, and in Fig.~\ref{fig:f4}-(f) with digits $9$ and $7$.
\begin{figure}[htbp]
\centering
\begin{tikzpicture}[zoomboxarray]

      \node [image node] { \includegraphics[scale=0.46]{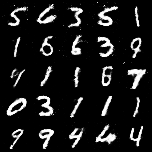} };
   
    \node[draw ] at (10, 2) {Epoch 85};

    \zoombox[magnification=1.2,color code=yellow]{0.6,0.8}

    \zoombox[magnification=1.2, color code=orange]{0.4,0.2}

    \zoombox[magnification=1.2,color code=lime]{0.8,0.4}
     
\end{tikzpicture}

\begin{tikzpicture}[zoomboxarray]

    \node [image node] { \includegraphics[scale=0.46]{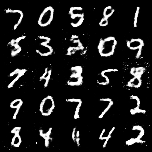} }; 
      \node[draw ] at (10, 2) {Epoch 92};


    \zoombox[magnification=1.2,color code=yellow]{0.8,0.78}

    \zoombox[magnification=1.2, color code=orange]{0.2,0.2}

    \zoombox[magnification=1.2,color code=lime]{0.8,0.2}

\end{tikzpicture}

\begin{tikzpicture}[zoomboxarray]
    
    \node [image node] { \includegraphics[scale=0.46]{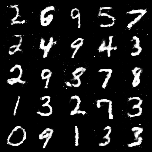} }; 
      \node[draw ] at (10, 2) {Epoch 96};

    \zoombox[magnification=1.28,color code=yellow]{0.5,0.78}

    \zoombox[magnification=1.28, color code=orange]{0.7,0.39}

    \zoombox[magnification=1.2,color code=lime]{0.8,0.2}

\end{tikzpicture}

\caption{Zoom in on images generated with GM-GAN at different epochs.}
\label{fig:f4}
\end{figure}

GM-GAN complexity is also reduced throughout the equivariance property by eliminating the need to learn symmetries. This is illustrated by reducing MNIST training dataset by a half and comparing generated images at epoch $42$. GM-GAN results (Fig.~\ref{subfig:m30k}) show again better image quality and high variations of generated digits in comparison to GAN (Fig.~\ref{subfig:g30k}). Results highlight the importance of equivariance in morphological operators, turning out to dataset reduction without significantly impacting generation results (see Fig.~\ref{subfig:m42} for GM-GAN and Fig.~\ref{subfig:g42} for images generated at the same epoch using the hole dataset).\\
\begin{figure}[htbp]
  \centering
    \begin{subfigure}{0.23\textwidth}
    \includegraphics[width=\linewidth]{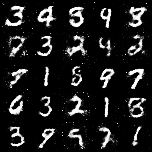}
    \caption{GM-GAN(1/2)}
    \label{subfig:m30k36}
  \end{subfigure}\hspace{1em}%
  \begin{subfigure}{0.23\textwidth}
    \includegraphics[width=\linewidth]{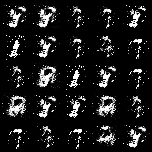}
    \caption{GAN($1/2$)}
    \label{subfig:g30k36}
  \end{subfigure}\hspace{1em}%
  \begin{subfigure}{0.23\textwidth}
    \includegraphics[width=\linewidth]{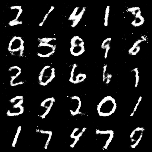}
    \caption{GM-GAN}
    \label{subfig:m36}
  \end{subfigure}\hspace{1em}%
  \begin{subfigure}{0.23\textwidth}
    \includegraphics[width=\linewidth]{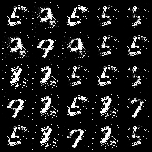}
    \caption{GAN}
    \label{subfig:g36}
  \end{subfigure}
  \\
  \begin{subfigure}{0.23\textwidth}
    \includegraphics[width=\linewidth]{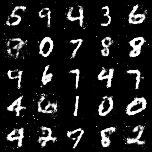}
    \caption{GM-GAN(1/2)}
    \label{subfig:m30k}
  \end{subfigure}\hspace{1em}%
  \begin{subfigure}{0.23\textwidth}
    \includegraphics[width=\linewidth]{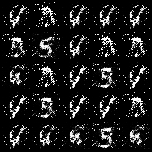}
    \caption{GAN($1/2$)}
    \label{subfig:g30k}
  \end{subfigure}\hspace{1em}%
  \begin{subfigure}{0.23\textwidth}
    \includegraphics[width=\linewidth]{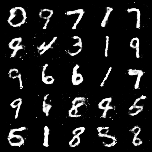}
    \caption{GM-GAN}
    \label{subfig:m42}
  \end{subfigure}\hspace{1em}%
  \begin{subfigure}{0.23\textwidth}
    \includegraphics[width=\linewidth]{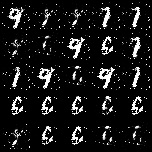}
    \caption{GAN}
    \label{subfig:g42}
  \end{subfigure}
    \caption{Comparison between generated images at epoch $36$ (top line) and $42$ (bottom line) using $-$ Half MNIST dataset: (\ref{subfig:m30k36},\ref{subfig:m30k}) MG-GAN, (\ref{subfig:g30k36},\ref{subfig:g30k}) GAN. Whole MNIST dataset: (\ref{subfig:m36},\ref{subfig:m42}) MG-GAN, (\ref{subfig:g36},\ref{subfig:g42}) GAN.}
  \label{fig:c4fig9}
\end{figure} 
\\
To highlight again the usefulness of morphological equivariant operators, we apply both GM-GAN and GAN models on RotoMNIST; generated images are displayed in Fig.~\ref{fig:f5}. It can be seen in results obtained with GM-GAN from epoch $70$ to $95$ (Figs.~\ref{subfig:mr70}, \ref{subfig:mr71}, \ref{subfig:mr72}, \ref{subfig:mr73}, and \ref{subfig:mr74}) that digits are clearly identifiable and far better than those generated with GAN at the same epochs (Figs.~\ref{subfig:ganr70}, \ref{subfig:ganr71}, \ref{subfig:ganr72}, \ref{subfig:ganr73}, and \ref{subfig:ganr74}) where digits are barely formed. The same is noticed with GM-GAN from epoch $100$ to $120$ (Figs.~\ref{subfig:mr90}, \ref{subfig:mr91}, \ref{subfig:mr92}, \ref{subfig:mr93}, and \ref{subfig:mr94}), in comparison with GAN for the last $5$ epochs (Figs.~\ref{subfig:ganr195}, \ref{subfig:ganr196}, \ref{subfig:ganr197}, \ref{subfig:ganr198}, and \ref{subfig:ganr199}). This demonstrates that GM-GAN is more suitable for data under rotation transformations, and highlights one more time the importance of equivariance for generating satisfactory results under various transformations.

\begin{figure}[htbp]
  \centering
  \begin{subfigure}{0.22\textwidth}
    \includegraphics[width=\linewidth]{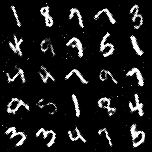}
    \caption{GM-GAN: $75$}
    \label{subfig:mr70}
  \end{subfigure}\hspace{1em}%
  \begin{subfigure}{0.22\textwidth}
    \includegraphics[width=\linewidth]{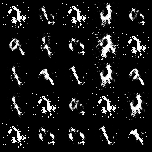}
    \caption{GAN: $75$}
    \label{subfig:ganr70}
  \end{subfigure}\hspace{1em}%
  \begin{subfigure}{0.22\textwidth}
    \includegraphics[width=\linewidth]{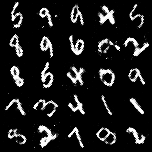}
    \caption{GM-GAN: $100$}
    \label{subfig:mr90}
  \end{subfigure}\hspace{1em}%
  \begin{subfigure}{0.22\textwidth}
    \includegraphics[width=\linewidth]{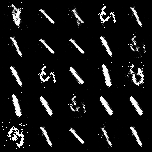}
    \caption{GAN: $195$}
    \label{subfig:ganr195}
  \end{subfigure}
  \\ 
  \begin{subfigure}{0.22\textwidth}
    \includegraphics[width=\linewidth]{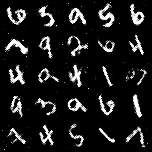}
    \caption{GM-GAN: $80$}
    \label{subfig:mr71}
  \end{subfigure}\hspace{1em}%
  \begin{subfigure}{0.22\textwidth}
    \includegraphics[width=\linewidth]{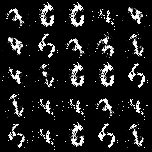}
    \caption{GAN: $80$}
    \label{subfig:ganr71}
  \end{subfigure}\hspace{1em}%
  \begin{subfigure}{0.22\textwidth}
    \includegraphics[width=\linewidth]{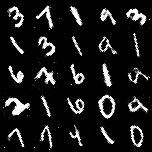}
    \caption{GM-GAN: $105$}
    \label{subfig:mr91}
  \end{subfigure}\hspace{1em}%
  \begin{subfigure}{0.22\textwidth}
    \includegraphics[width=\linewidth]{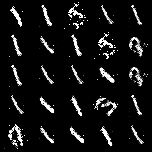}
    \caption{GAN: $196$}
    \label{subfig:ganr196}
  \end{subfigure}
  \\
  \begin{subfigure}{0.22\textwidth}
    \includegraphics[width=\linewidth]{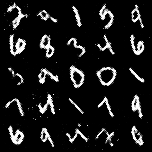}
    \caption{GM-GAN: $85$}
    \label{subfig:mr72}
  \end{subfigure}\hspace{1em}%
    \begin{subfigure}{0.22\textwidth}
    \includegraphics[width=\linewidth]{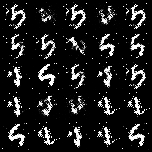}
    \caption{GAN: $85$}
    \label{subfig:ganr72}
  \end{subfigure}\hspace{1em}%
      \begin{subfigure}{0.22\textwidth}
    \includegraphics[width=\linewidth]{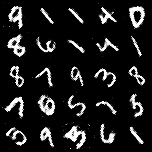}
    \caption{GM-GAN: $110$}
    \label{subfig:mr92}
  \end{subfigure}\hspace{1em}%
      \begin{subfigure}{0.22\textwidth}
    \includegraphics[width=\linewidth]{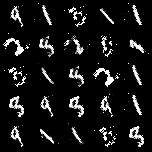}
    \caption{GAN: $197$}
    \label{subfig:ganr197}
  \end{subfigure} 
  \\ 
  \begin{subfigure}{0.22\textwidth}
    \includegraphics[width=\linewidth]{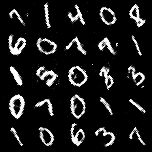}
    \caption{GM-GAN:$90$}
    \label{subfig:mr73}
  \end{subfigure}\hspace{1em}%
  \begin{subfigure}{0.22\textwidth}
    \includegraphics[width=\linewidth]{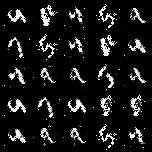}
    \caption{GAN: $90$}
    \label{subfig:ganr73}
  \end{subfigure}\hspace{1em}%
  \begin{subfigure}{0.22\textwidth}
    \includegraphics[width=\linewidth]{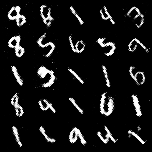}
    \caption{GM-GAN: $115$}
    \label{subfig:mr93}
  \end{subfigure}\hspace{1em}%
  \begin{subfigure}{0.22\textwidth}
    \includegraphics[width=\linewidth]{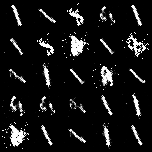}
    \caption{GAN: $198$}
    \label{subfig:ganr198}
  \end{subfigure}
  \\ 
  \begin{subfigure}{0.22\textwidth}
    \includegraphics[width=\linewidth]{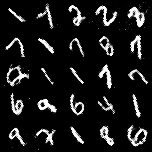}
    \caption{GM-GAN: $95$}
    \label{subfig:mr74}
  \end{subfigure}\hspace{1em}%
  \begin{subfigure}{0.22\textwidth}
    \includegraphics[width=\linewidth]{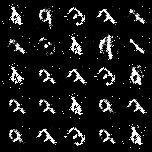}
    \caption{GAN: $95$}
    \label{subfig:ganr74}
  \end{subfigure}\hspace{1em}%
  \begin{subfigure}{0.22\textwidth}
    \includegraphics[width=\linewidth]{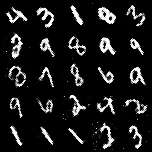}
    \caption{GM-GAN: $120$}
    \label{subfig:mr94}
  \end{subfigure}\hspace{1em}%
  \begin{subfigure}{0.22\textwidth}
    \includegraphics[width=\linewidth]{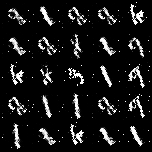}
    \caption{GAN: $199$}
    \label{subfig:ganr199}
  \end{subfigure}
  \\ 

  \caption{Image generation using RotoMNIST: GM-GAN vs. GAN.}
  \label{fig:f5}
\end{figure}
Quantitative evaluations are provided using the Fréchet Inception Distance (FID). A low FID indicates a high similarity between generated and real data, corresponding to good generation quality. In Fig.~\ref{fid}, we present the FID curves of both models over epochs (taking FID of generated images at intervals of $10$ epochs) on both MNIST and RotoMNIST datasets. It can be seen that starting from epoch $40$, FIDs of GM-GAN generated results are significantly lower than ones generated using GAN, which confirms the qualitative results discussed just above.
\begin{figure}[!] 
    \centering
     \begin{tabular}{cccc}
    \includegraphics[scale=0.3]{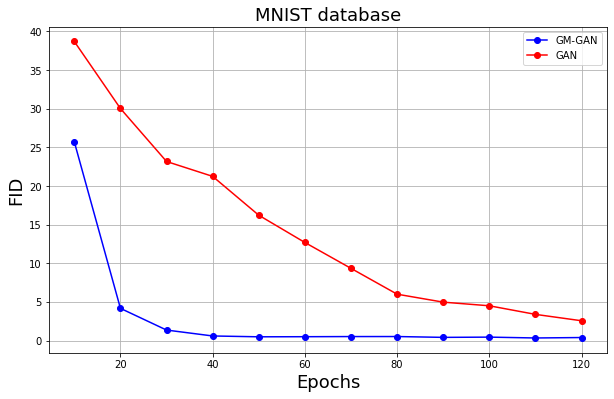} 
    &
    \includegraphics[scale=0.3]{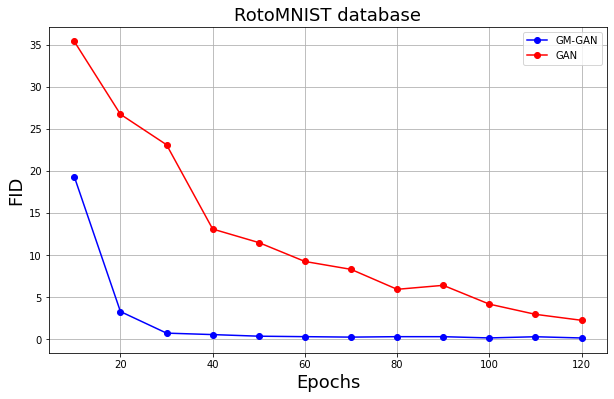} \\
    (a)
    &
    (b)
    \end{tabular}
    \caption{FID using GM-GAN vs. GAN with: (a) MNIST. (b) RotoMNIST.}
 \label{fid}
\end{figure} 
\renewcommand{\arraystretch}{1} 
\begin{table}[htbp]
  \centering
  \begin{tabular}{|c|c|c|}
    \hline
    \multirow{2}{*}{Models} & \multicolumn{2}{c|}{Metrics} \\
    \cline{2-3}
    & KL divergence & FID metric\\
    \hline
    GM-GAN & $0.95$ & $0.93$ \\
    \hline
    GAN & $1.07$ & $15.55$ \\
    \hline
  \end{tabular}
  \caption{KL and FID metrics for GM-GAN an GAN.}
  \label{tab:metric}
\end{table}

As seen in Table~\ref{tab:metric}, GM-GAN exhibits lower KL and much lower FID than GAN, suggesting that data generated with GM-GAN is closer, more realistic and trustworthy to the real data in terms of feature distribution.

\section{Conclusion and perspectives}\label{sect7}

We have proposed here a geometric generative GM-GAN model based on PDE-G-CNNs and built from derived equivariant morphological operators and geometric image processing techniques. The proposed equivariant morphological PDE layers are composed of multiscale dilations and erosions without any need to approximate convolutions kernels, and meanwhile, group symmetries are defined on Lie groups allowing a geometrical interpretability of GM-GAN with left invariance properties. As shown by preliminary results on MNIST and RotoMNIST datasets, preliminary qualitative and quantitative results show noticeable improvements compared classical GAN. Indeed, thin image features are better extracted by accounting intrinsic geometric features at multiscale levels, and the network complexity is reduced. The proposed approach can be extended to various generative models. Future works include applying GM-GAN on other datasets and designing fully equivariant generative models entirely based on PDE-G-CNNs.

\newpage
\bibliographystyle{splncs04}
\bibliography{Le_Thier_bib,biblio}

\end{document}